\documentclass{article}

\usepackage{microtype}
\usepackage{graphicx}
\usepackage{subfigure}
\usepackage{booktabs}

\usepackage{hyperref}

\usepackage[accepted]{icml2024}

\usepackage{amsmath}
\usepackage{amssymb}
\usepackage{mathtools}
\usepackage{amsthm}

\usepackage{xcolor}

\usepackage[capitalize,noabbrev]{cleveref}

\theoremstyle{plain}
\newtheorem{theorem}{Theorem}[section]
\newtheorem{proposition}[theorem]{Proposition}

\theoremstyle{definition}

\theoremstyle{remark}

\usepackage[textsize=tiny]{todonotes}

\icmltitlerunning{\smash{Muon Operates Beyond the Edge of Stability}}

\begin{document}

\twocolumn[
\icmltitle{Muon Operates Beyond the Edge of Stability}

\icmlsetsymbol{equal}{*}

\begin{icmlauthorlist}
\icmlauthor{Niko Čović}{equal,eth}
\icmlauthor{Baraq Lipshitz}{equal,eth}
\icmlauthor{Alexander Padula}{equal,eth}
\icmlauthor{Stefan Višnjić}{equal,eth}
\end{icmlauthorlist}

\icmlaffiliation{eth}{Department of Informatics, ETH Zurich, Switzerland}

\icmlcorrespondingauthor{Alexander Padula}{apadula@ethz.ch}
\icmlkeywords{Sharpness, Edge of Stability, Machine Learning, Deep Learning}

\vskip 0.3in
]

\printAffiliationsAndNotice{\icmlEqualContribution}

\begin{abstract}

We study the Muon optimizer using sharpness and provide the first empirical evidence that the Flat Minima Hypothesis also applies to Muon. Further, we derive Muon's implicit preconditioner to analyze its potential Edge of Stability behavior. Experiments on CIFAR-10 show sharpness strongly correlates with generalization gap within each optimizer (Pearson $r=0.45-0.69, p\leq10^{-4}$), confirming the hypothesis holds for Muon similarly to Adam and SGD. However, Muon operates above the classical sharpness-based Adaptive Edge of Stability threshold that bounds diagonally-preconditioned optimizers like Adam and RMSprop, yet remains stable. Our findings suggest Muon operates in a fundamentally different regime than diagonally-adaptive methods. Defining the exact stability boundary for Muon remains an open question. All code is available at \url{https://github.com/NikoCovic/DL-project}.
\end{abstract}

\section{Introduction}
\label{introduction}

Muon is a matrix-structured optimizer recently introduced by \citet{jordan2024muon}. Recent work has empirically observed that it converges in fewer steps than first-order optimizers such as Adam \cite{kingma2014adam} and Stochastic Gradient Descent (SGD) while incurring less computational overhead than related second-order optimizers like Shampoo \cite{gupta2018shampoo}. While Muon has generated excitement for converging to high accuracy solutions in less wall-clock time than competing methods, the mechanisms underlying this behavior remain a subject of ongoing research. 

In this paper, we aim to characterize Muon's behavior in terms of the sharpness of its solutions. We are the first to empirically observe the relationship between sharpness and generalization predicted by the Flat Minima Hypothesis for Muon's solutions. We are also the first to attempt to verify whether Muon is bounded by an Edge of Stability in the same way as first-order optimizers such as Gradient Descent (GD) and Adam. We present a derivation for Muon's preconditioner and empirically show that Muon operates outside of the Edge of Stability as found in diagonally preconditioned optimizers such as Adam.

\subsection{Flat Minima Hypothesis}

 The Flat Minima Hypothesis \cite{hochreiter1997flat} \cite{keskar2017large} predicts that sharpness should correlate with the model's generalization gap — the difference between training and validation accuracy. It relies on the intuition that the validation loss landscape can be approximated as a stochastically perturbed version of the training landscape. Subsequent work has highlighted the limitations of this hypothesis. Most notably, \citet{dinh2017sharp} demonstrated that raw sharpness is not scale-invariant; one can alter it via re-parametrization without changing the model's function. This occurs frequently when comparing models across different optimizers. In response, methods such as Adaptive Sharpness-Aware Minimization (ASAM) \cite{kwon2021asam} have been proposed to define a scale-invariant geometric measure of a minimum's size. By validating the Flat Minima Hypothesis for Muon's solutions, we investigate whether there exist sharpness-hacking reparametrizations beyond scale invariance that a novel optimizer like Muon might inadvertently favor.

\subsection{Edge of Stability}

Although much work exists formalizing Muon's convergence \cite{shen2025muonconvergence, chen2025muonspectralnorm, sfyraki2025lionsmuons, sato280140878muonconvergencebatch, chang2025muonconvergencebeyond}, none attempts to look for an explanation using the Edge of Stability, a common phenomenon present in state-of-the-art optimizers such as Adam. This phenomenon was originally formalized in \cite{cohen2021gdeos} for full-batch GD. The authors identify that, throughout training, GD's sharpness seems to hover around a stability threshold value. They reason that divergence does not occur due to higher order Taylor terms which implicitly regularize the sharpness.

In \cite{cohen2024adaptivegradientmethodsedge} the authors extend the classical EoS to adaptive gradient methods (i.e. algorithms with implicit adaptive preconditioners) and coin the term Adaptive EoS (AEoS). Their work finds that the phenomenon still occurs in the sharpness of the effective Hessian $\lambda_\text{max}(P_tH_t)$, where $P_t$ is the implicit preconditioner at time $t$. Another important contribution of the paper is that they reason that the adaptive preconditioner allows the optimizer to reach sharper regions of the loss landscape. However, this work focuses only on diagonally preconditioned algorithms such as Adam and RMSprop \cite{tieleman2017divide}. 

In our experiments, breaking this assumption using Muon’s non-diagonal preconditioner eliminates the classical AEoS phenomenon. While the mechanism that keeps Muon in a stable regime remains unclear, we hypothesize that AEoS disappears because Muon’s updates are poorly aligned with the eigenbasis of the effective Hessian $P_tH_t$.

\section{Models and Methods}
\subsection{Optimizers}

We formalize all of the optimizers we use in our experiments in Appendix \ref{app:optims}. 

To determine the quality of sharpness as a generalization metric, we use \textsc{DecoupledMuon}, \textsc{NormalizedMuon}, \textsc{CoupledSGD} and \textsc{CoupledAdam}.

To quantify the AEoS, we compare \textsc{VanillaMuon} to \textsc{VanillaAdam} and \textsc{RMSprop}.

\subsection{Sharpness}
For parameters $w_t$, let $H_t := \nabla_w^2 \mathcal{L}(w_t)$ denote the Hessian of the training loss.

\textbf{Raw sharpness:} We define raw (Hessian) sharpness as
\begin{equation}
\mathcal{S}_{\text{raw}}(w_t) := \lambda_{\max}(H_t).
\end{equation}
Raw sharpness suffers from sensitivity to loss-preserving parameter rescaling, which is particularly an issue when comparing sharpness across solutions derived from different optimizers \cite{kwon2021asam}. Two parameterizations with identical functions can yield arbitrarily different sharpness values.

\textbf{Adaptive sharpness:} To avoid conflating optimizer behavior with irrelevant scaling effects, we additionally report adaptive sharpness \cite{kwon2021asam}, defined as

\begin{equation}
\mathcal S_{\text{adapt}}(w)
:= \max_{\|T_w^{-1}\epsilon\|\le\rho}
\big[\mathcal{L}(w+\epsilon)-\mathcal{L}(w)\big],
\end{equation}

where $T_w$ is a normalization operator depending on the current parameters. This defines a perturbation set in scale-normalized space, making the measure invariant to loss-preserving parameter rescaling. Unlike classical sharpness it should be comparable across different optimizers.

\textbf{Effective sharpness}: To measure the AEoS, raw sharpness is not sufficient except in the GD setting. In adaptive gradient optimizers the dynamics are governed by the effective hessian $P_tH_t$, where $P_t$ is the implicit preconditioner of the optimizer. We define the effective sharpness as:

\begin{equation}
\mathcal{S}_{\text{eff}}(w_t) = \lambda_{\text{max}}(P_tH_t)
\end{equation}

\subsection{Generalization Benchmark Methodology}
We base our first experimental setup on a variant of Airbench \cite{jordan2024airbench}, an optimized script that trains a VGG-like CNN to 94\% accuracy on the CIFAR-10 \cite{cifar10} image classification dataset.

We modify the Airbench script by (i) adding a callback function called after each epoch to measure sharpness, (ii) supporting training with a fixed learning rate in addition to the original Linear Decay Scheduler (LDS), and (iii) implementing \textsc{DecoupledMuon}, \textsc{CoupledAdam}, and \textsc{CoupledSGD} in addition to the \textsc{NormalizedMuon} implementation from the original script. 

For each optimizer, we perform an extensive hyperparameter sweep using Bayesian optimization with a Gaussian-process (GP) surrogate model. We consider both a fixed learning rate and an LDS schedule. Each sweep consists of 128-512 trials where the GP posterior selects hyperparameters to maximize validation accuracy. All reported results are using the best hyperparameters found by this procedure. 

Starting with highly optimized architectures and hyperparameters lends weight to the empirical observations we make regarding Muon's performance compared to Adam and SGD, and allows us to measure sharpness in a practical setting.

We train each run for 16 epochs to ensure convergence. After the final epoch, we report per-optimizer mean validation accuracy, generalization gap, and both raw and adaptive sharpness. For each optimizer, we compute correlation between the sharpness and the generalization gap.

\subsection{Edge of Stability}

To investigate whether Muon operates at the AEoS, we conduct experiments analogous to those in \cite{cohen2024adaptivegradientmethodsedge}. \textsc{VanillaMuon}($\eta, \mu$) (and other gradient-based algorithms) constructs an implicit adaptive preconditioner $P_t$ as we formalize in Appendix \ref{app:muon_preconditioner}, differing from \textsc{VanillaAdam}($\eta, \beta_1, \beta_2$) in that the preconditioner of \textsc{VanillaMuon}($\eta, \mu$) is non-diagonal. Using this we compute the sharpness of the effective Hessian $P_tH_t$. Note that this matrix has the same eigenvalues as the symmetric matrix $P_t^{1/2}H_tP_t^{1/2}$, so we opt to compute the sharpness of this matrix instead.

We measure $\mathcal{S}_\text{raw}$ and $\mathcal{S}_\text{eff}$ in each epoch of multiple full batch training runs for \textsc{VanillaMuon}($\eta, \mu$), \textsc{VanillaAdam}($\eta, \beta_1, \beta_2$) and \textsc{RMSprop}($\eta, \alpha$), each with fixed hyperparameters and multiple fixed learning rates $\eta$. Note that we restrict our analysis to algorithms without weight decay. This is because weight decay only introduces a slight difference in the constructed preconditioner when coupled, and only a slight difference in the loss when decoupled. As for \textsc{NormalizedMuon}($\eta, \mu$), the update operates in a different regime, where the \textsc{VanillaMuon}($\eta, \mu$) updates are projected onto a unit Frobenius ball. We believe the analysis of this algorithm would yield similar results, so we keep it out for simplicity.

We use a fully connected network with the tanh activation, a width of 200, and 5 hidden layers. Each run was executed for 1000 epochs on a subset of the CIFAR-10 dataset, which contained 200 samples per class and was preprocessed with mean normalization, using a train-validation split of 0.8-0.2. Finally, each model was trained on the MSE loss, as it demonstrates the AEoS behavior more clearly \cite{cohen2021gdeos}.

To construct the preconditioner of \textsc{VanillaMuon}, we use our derivation from Appendix \ref{app:muon_preconditioner}. Every epoch we store the new and previous weights $W_{t+1}^{(l)}$ and $W_t^{(l)}$ for each layer. This allows us to compute the update matrix as $U\Sigma'V^T = (W_t^{(l)} - W_{t+1}^{(l)})/\eta$ and construct the respective block of the preconditioner as $D_t^{(l)} = U \Sigma^{-1} \Sigma' U^T$, where $M_t^{(l)} = U \Sigma V^T$ is the SVD of the momentum. The preconditioner for the other optimizers is simpler, as it only requires a diagonal equal to the update at that step.

\section{Results}

\subsection{Generalization Benchmark Results}

First, \autoref{fig:fixed_sharp_gap} shows that under a fixed learning rate, sharpness strongly correlates with generalization gap within each optimizer. For raw sharpness, Pearson correlations range from 0.27 (Adam) to 0.49 (SGD), all statistically significant. Using adaptive sharpness further strengthens this relationship, with correlations increasing to 0.45–0.69 across optimizers (all $p<10^{-4}$). However, these trends do not transfer across optimizers: raw sharpness varies substantially between methods without corresponding changes in gap, indicating weak cross-optimizer predictiveness.

Beyond correlations, we observe a new pattern in sharpness magnitude. Under a fixed learning rate, Muon exhibits substantially lower adaptive sharpness (ASAM = 0.82 ± 0.229) than SGD (2.06 ± 0.911) and Adam (1.93 ± 0.626). Interestingly, this trend reverses under LDS, where Muon becomes sharper (0.71 ± 0.097) than SGD (0.32 ± 0.072) and Adam (0.56 ± 0.074).

Turning to performance, our results align with prior work: Muon variants consistently achieve higher validation accuracy. Under LDS, Normalized Muon reaches 0.94 ± 0.002, outperforming SGD and Adam (both 0.92 ± 0.002). In contrast, fixed learning rates degrade performance for all optimizers (e.g., Normalized Muon: 0.94 → 0.90; SGD: 0.92 → 0.80).

Finally, despite these differences, the mean generalization gap remains stable across schedulers, staying within 0.04–0.08 for all optimizers.

For completeness, we report the means and standard deviations of the aforementioned accuracy and sharpness measures in \autoref{tbl:benchmark} in the Appendix. We additionally note that, under a scheduled learning rate, we do not observe the correlation predicted by the Flat Minima Hypothesis; see \autoref{fig:scheduled_sharp_gap} in the Appendix.

\begin{figure*}[ht]
\vskip 0.2in
\begin{center}
\centering
\includegraphics[width=0.98\textwidth]{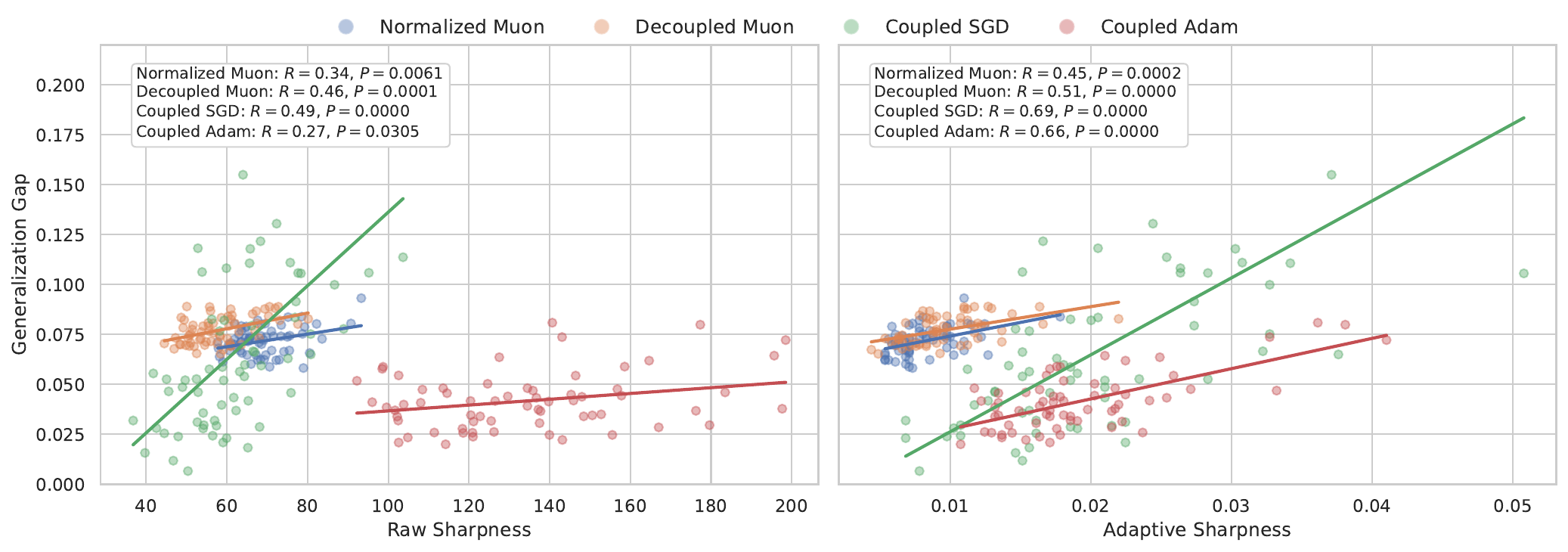}
\caption{Per-optimizer correlation of raw/adaptive sharpness with generalization gap using a fixed LR. Each point is the final checkpoint for 1 of 64 runs per optimizer. }
\label{fig:fixed_sharp_gap}
\end{center}
\vskip -0.2in
\end{figure*}

\subsection{Edge of Stability}

In Figure \ref{fig:muon_eos} it is clearly visible that \textsc{VanillaMuon}($\eta, \mu$) seems to operate beyond the classical stability threshold. However, it does not seem to diverge as apparent from the loss trajectory. Interestingly, the raw sharpness appears to exhibit similar behavior to \textsc{Adam}($\eta, \beta_1, \beta_2$) and \textsc{RMSprop}($\eta, \alpha$) up to convergence (see Figure \ref{fig:adam_eos}), however past convergence the sharpness of \textsc{VanillaMuon}($\eta, \mu$) starts dropping rapidly, a phenomenon not present in the other algorithms.

\begin{figure*}[ht]
\vskip 0.2in
\begin{center}
\begin{subfigure}
    \centering
    \includegraphics[width=0.66\columnwidth]{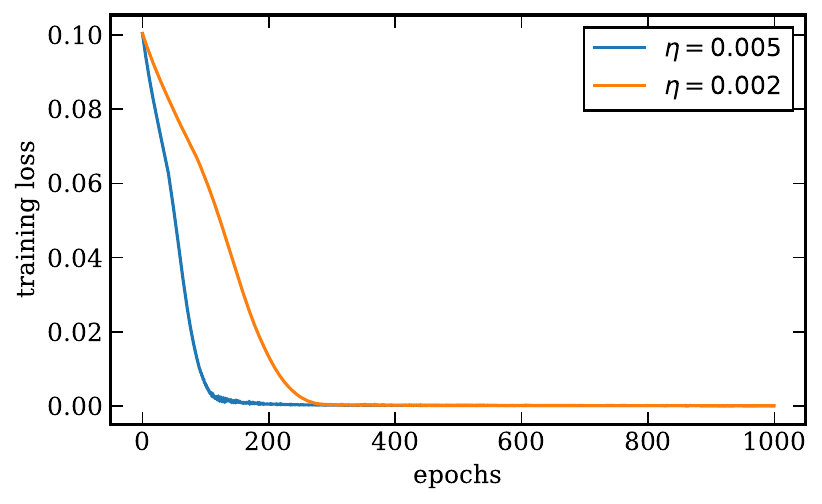}
\end{subfigure}
\begin{subfigure}
    \centering
    \includegraphics[width=0.66\columnwidth]{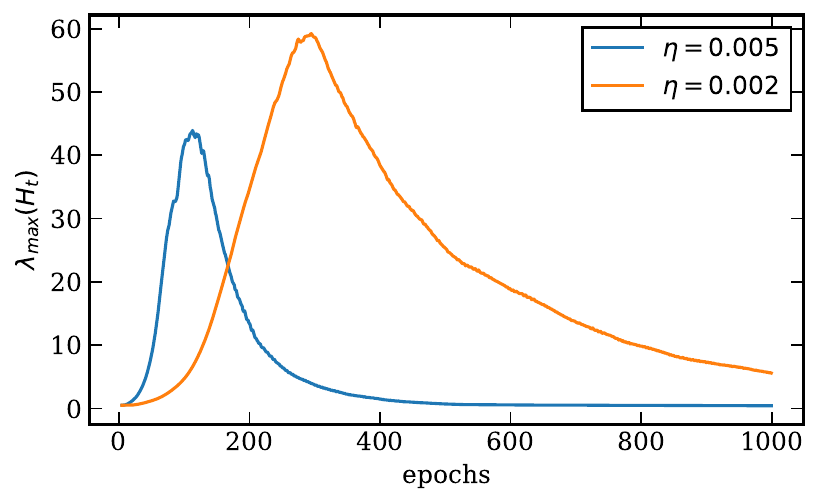}
\end{subfigure}
\begin{subfigure}
    \centering
    \includegraphics[width=0.66\columnwidth]{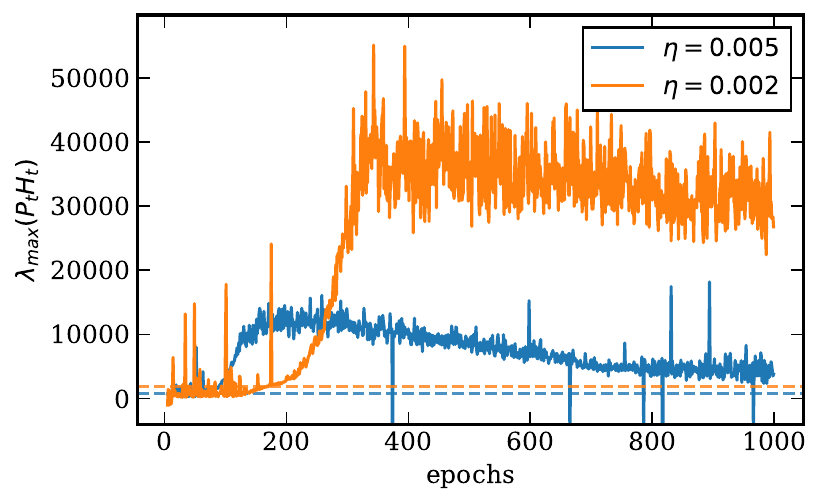}
\end{subfigure}
\caption{The training loss (left), raw sharpness (middle) and Effective Sharpness (right) of \textsc{VanillaMuon}($\eta, \mu$) with $\mu = 0.9$ and two different $\eta$. The dashed lines in the right plot represent the threshold where it is expected for Muon to oscillate, computed as $(2 + 2\mu)/\eta$ \cite{goh2017momentum, cohen2024adaptivegradientmethodsedge}. Clearly, \textsc{VanillaMuon}($\eta, \mu$) operates beyond this threshold.}
\label{fig:muon_eos}
\end{center}
\vskip -0.2in
\end{figure*}

\section{Discussion}

\subsection{Flat Minima Hypothesis}

In aggregate, we find that the Flat Minima Hypothesis holds for solutions obtained with Muon to a similar extent as for those obtained with Adam and SGD. The fact that, in \autoref{fig:fixed_sharp_gap}, runs across different optimizers remain broadly correlated with generalization suggests that scaling bias largely explains the distinct clustering observed in the Raw Sharpness plot. Nevertheless, the Adaptive Sharpness plot continues to exhibit optimizer-specific clusters. If Adaptive Sharpness fully accounted for all parametrization bias, no such clustering should be expected, as the sharpness measure would depend solely on the generalization gap and noise. Muon’s cluster, for instance, exhibits lower Adaptive Sharpness, yet this does not correspond to a comparably reduced generalization gap. Assuming that the Flat Minima Hypothesis holds for some (as yet unidentified) flatness measure, our results indicate that, while Adaptive Sharpness is effective at mitigating scale-related bias, it is not a completely parametrization-invariant measure of flatness.

\subsection{Edge of Stability}

As apparent from Figure \ref{fig:muon_eos}, \textsc{VanillaMuon} seems to operate beyond the classical AEoS, but still does not diverge. Empirically, \textsc{VanillaMuon} may operate within a different regime compared to \textsc{Adam} and \textsc{RMSprop}. We hypothesize that this is caused by those algorithms constructing a diagonal preconditioner, while \textsc{VanillaMuon} constructs a block-diagonal one. Since stability dynamics are controlled by $P_tH_t$, we speculate that the large sharpness of Muon stems from the misalignment of the update directions with the eigenbasis of $P_tH_t$. This is not a problem for diagonal optimizers, as the alignment is expected to be high. Sharpness could therefore be overestimating the "stability" of \textsc{VanillaMuon} and should be replaced with a different metric. However, we do not explore this in our work so it remains an open question.

A surprising observation also comes from the raw sharpness found in \textsc{VanillaMuon}'s trajectory, seen in the middle plot of Figure \ref{fig:muon_eos}. Like the raw sharpness of \textsc{VanillaAdam} and \text{RMSprop}, it rises until convergence, as expected given the assumptions of \cite{cohen2024adaptivegradientmethodsedge}. However, post convergence \textsc{VanillaMuon} seems to drastically reduce its raw sharpness. This does not happen with \textsc{VanillaAdam} and \textsc{RMSprop}, where the raw sharpness seems to plateau post convergence. This is another strong hint that \textsc{VanillaMuon} operates within a different regime. 

\section{Summary}

We were able to reproduce the Flat Minima Hypothesis within optimizers under a realistic setting and show that it also applies to Muon. We show, however, that even using Adaptive Sharpness the hypothesis does not explain all the variance in the generalization gap.

We show that the AEoS does not appear with Muon in terms of sharpness as it does with diagonally preconditioned optimizers such as Adam and RMSprop. We hypothesize that this stems from the misalignment of the preconditioners eigenvalues with the update directions, however we do not explore this in our work and leave this as an open question.

\bibliography{main}

\begin{thebibliography}{20}
\providecommand{\natexlab}[1]{#1}
\providecommand{\url}[1]{\texttt{#1}}
\expandafter\ifx\csname urlstyle\endcsname\relax
  \providecommand{\doi}[1]{doi: #1}\else
  \providecommand{\doi}{doi: \begingroup \urlstyle{rm}\Url}\fi

\bibitem[Chang et~al.(2025)Chang, Liu, and
  Yuan]{chang2025muonconvergencebeyond}
Chang, D., Liu, Y., and Yuan, G.
\newblock On the convergence of {Muon} and beyond.
\newblock \emph{arXiv preprint arXiv:2509.15816}, 2025.

\bibitem[Chen et~al.(2025)Chen, Li, and Liu]{chen2025muonspectralnorm}
Chen, L., Li, J., and Liu, Q.
\newblock {Muon} optimizes under spectral norm constraints.
\newblock \emph{arXiv preprint arXiv:2506.15054}, 2025.

\bibitem[Cohen et~al.(2021)Cohen, Kaur, Li, Kolter, and
  Talwalkar]{cohen2021gdeos}
Cohen, J.~M., Kaur, S., Li, Y., Kolter, J.~Z., and Talwalkar, A.
\newblock Gradient descent on neural networks typically occurs at the edge of
  stability.
\newblock \emph{arXiv preprint arXiv:2103.00065}, 2021.

\bibitem[Cohen et~al.(2024)Cohen, Ghorbani, Krishnan, Agarwal, Medapati,
  Badura, Suo, Cardoze, Nado, Dahl, and
  Gilmer]{cohen2024adaptivegradientmethodsedge}
Cohen, J.~M., Ghorbani, B., Krishnan, S., Agarwal, N., Medapati, S., Badura,
  M., Suo, D., Cardoze, D., Nado, Z., Dahl, G.~E., and Gilmer, J.
\newblock Adaptive gradient methods at the edge of stability, 2024.
\newblock URL \url{https://arxiv.org/abs/2207.14484}.

\bibitem[Dinh et~al.(2017)Dinh, Pascanu, Bengio, and Bengio]{dinh2017sharp}
Dinh, L., Pascanu, R., Bengio, S., and Bengio, Y.
\newblock Sharp minima can generalize for deep nets.
\newblock In \emph{International Conference on Machine Learning (ICML)}, pp.\
  1019--1028. PMLR, 2017.

\bibitem[Goh(2017)]{goh2017momentum}
Goh, G.
\newblock Why momentum really works.
\newblock \emph{Distill}, 2\penalty0 (4):\penalty0 e6, 2017.

\bibitem[Gupta et~al.(2018)Gupta, Koren, and Singer]{gupta2018shampoo}
Gupta, V., Koren, T., and Singer, Y.
\newblock Shampoo: Preconditioned stochastic tensor optimization, 2018.
\newblock URL \url{https://arxiv.org/abs/1802.09568}.

\bibitem[Hochreiter \& Schmidhuber(1997)Hochreiter and
  Schmidhuber]{hochreiter1997flat}
Hochreiter, S. and Schmidhuber, J.
\newblock Flat minima.
\newblock \emph{Neural Computation}, 9\penalty0 (1):\penalty0 1--42, 1997.

\bibitem[Jordan(2024)]{jordan2024airbench}
Jordan, K.
\newblock 94\% on {CIFAR-10} in 3.29 seconds on a single {GPU}, 2024.
\newblock URL \url{https://arxiv.org/abs/2404.00498}.

\bibitem[Jordan et~al.(2024)Jordan, Jin, Boza, You, Cesista, Newhouse, and
  Bernstein]{jordan2024muon}
Jordan, K., Jin, Y., Boza, V., You, J., Cesista, F., Newhouse, L., and
  Bernstein, J.
\newblock {Muon}: An optimizer for hidden layers in neural networks, 2024.
\newblock URL \url{https://kellerjordan.github.io/posts/muon/}.

\bibitem[Keskar et~al.(2017)Keskar, Mudigere, Nocedal, Smelyanskiy, and
  Tang]{keskar2017large}
Keskar, N.~S., Mudigere, D., Nocedal, J., Smelyanskiy, M., and Tang, P. T.~P.
\newblock On large-batch training for deep learning: Generalization gap and
  sharp minima.
\newblock In \emph{International Conference on Learning Representations
  (ICLR)}, 2017.

\bibitem[Kingma(2014)]{kingma2014adam}
Kingma, D.~P.
\newblock Adam: A method for stochastic optimization.
\newblock \emph{arXiv preprint arXiv:1412.6980}, 2014.

\bibitem[Krizhevsky et~al.()Krizhevsky, Nair, and Hinton]{cifar10}
Krizhevsky, A., Nair, V., and Hinton, G.
\newblock {CIFAR-10} (canadian institute for advanced research).
\newblock URL \url{http://www.cs.toronto.edu/~kriz/cifar.html}.

\bibitem[Kwon et~al.(2021)Kwon, Kim, Park, and Choi]{kwon2021asam}
Kwon, J., Kim, J., Park, H., and Choi, I.~K.
\newblock {ASAM}: Adaptive sharpness-aware minimization for scale-invariant
  learning of deep neural networks.
\newblock In \emph{International conference on machine learning}, pp.\
  5905--5914. PMLR, 2021.

\bibitem[Liu et~al.(2025)Liu, Su, Yao, Jiang, Lai, Du, Qin, Xu, Lu, Yan,
  et~al.]{liu2025muonscalableLLM}
Liu, J., Su, J., Yao, X., Jiang, Z., Lai, G., Du, Y., Qin, Y., Xu, W., Lu, E.,
  Yan, J., et~al.
\newblock {Muon} is scalable for {LLM} training.
\newblock \emph{arXiv preprint arXiv:2502.16982}, 2025.

\bibitem[Sato et~al.(2025)Sato, Naganuma, and
  Iiduka]{sato280140878muonconvergencebatch}
Sato, N., Naganuma, H., and Iiduka, H.
\newblock Convergence bound and critical batch size of {Muon} optimizer.
\newblock \emph{arXiv preprint arXiv:2507.01598}, 2025.

\bibitem[Sfyraki \& Wang(2025)Sfyraki and Wang]{sfyraki2025lionsmuons}
Sfyraki, M.-E. and Wang, J.-K.
\newblock Lions and muons: Optimization via stochastic frank-wolfe.
\newblock \emph{arXiv preprint arXiv:2506.04192}, 2025.

\bibitem[Shen et~al.(2025)Shen, Huang, Huang, Shen, and
  Zhang]{shen2025muonconvergence}
Shen, W., Huang, R., Huang, M., Shen, C., and Zhang, J.
\newblock On the convergence analysis of {Muon}.
\newblock \emph{arXiv preprint arXiv:2505.23737}, 2025.

\bibitem[Tieleman \& Hinton(2017)Tieleman and Hinton]{tieleman2017divide}
Tieleman, T. and Hinton, G.
\newblock Divide the gradient by a running average of its recent magnitude.
  coursera: Neural networks for machine learning.
\newblock In \emph{Technical report}. University of Toronto, 2017.

\bibitem[Yao et~al.(2020)Yao, Gholami, Keutzer, and Mahoney]{yao2020pyhessian}
Yao, Z., Gholami, A., Keutzer, K., and Mahoney, M.~W.
\newblock {PyHessian}: Neural networks through the lens of the hessian.
\newblock In \emph{2020 IEEE international conference on big data (Big data)},
  pp.\  581--590. IEEE, 2020.

\end{thebibliography}
\bibliographystyle{icml2024}

\newpage
\appendix
\onecolumn
\section{Muon Preconditioner Derivation} \label{app:muon_preconditioner}

In this section, we formally derive the implicit preconditioner of \textsc{VanillaMuon}($\eta, \mu$). We aim to rewrite the update in terms of the collected vectorized network weights $w_t = \begin{bmatrix}\text{vec}(W_t^{(1)})^T & \dots & \text{vec}(W_t^{(L)})^T\end{bmatrix}^T$ and a preconditioner $P_{t}$ as $w_{t+1} = w_t - \eta P_{t} m_{t}$, where $m_{t} = \begin{bmatrix}\text{vec}(M_{t}^{(1)})^T & \dots & \text{vec}(M_{t}^{(L)})^T\end{bmatrix}^T$ is the vectorized concatenated momentum. We begin by deriving a per-layer preconditioner and extend it to the entire network. Let $\mathcal{W}_t = \{W_t^{(l)} | 1 \le l \le L\}$ represent the set of all rectangular weight matrices optimized by Muon at time $t$. The authors in \cite{liu2025muonscalableLLM} have already shown that Muon uses Newton-Schulz to approximate $UV^T$ where $M_{t}^{(l)} = U \Sigma V^T$ is the SVD of the momentum. The implicit preconditioner then becomes $(M_{t}^{(l)}M_{t}^{(l)T})^{-1/2}$. Here, we derive the preconditioner assuming the approximate orthogonalization using Newton-Schulz.

\begin{proposition}
The Newton-Schulz update of \textsc{VanillaMuon}($\eta, \mu$) $O_{t}^{(l)}$ can be written as $O_{t}^{(l)} = U \Sigma' V^T$, where $M_{t}^{(l)} = U \Sigma V^T$ is the SVD of the corresponding momentum matrix and $\Sigma'$ is some diagonal matrix.
\end{proposition}

\begin{proof}
Let $t > 0$ and $1 \le l \le L$ be an arbitrary time and weight matrix index. Let $M_{t}^{(l)} = U \Sigma V^T$ be the singular value decomposition of the momentum matrix.

The update step of \textsc{VanillaMuon($\eta, \mu$)} is $W_{t+1}^{(l)} = W_t^{(l)} - \eta O_{t}^{(l)}$, where $\eta$ is the learning rate and $O_{t}^{(l)} = \text{Newton-Schulz5}(M_t^{(l)}; a,b,c)$ for some coefficients $a, b, c$. 

Let $X_{k+1} = aX_k + b(X_kX_k^T)X_k + c (X_kX_k^T)^2X_k$ be a single Newton-Schulz5 iteration for some matrix $X_k$. By setting $X_0 = M_{t}^{(l)}$ and inserting the SVD of $M_{t}^{(l)}$ into the iteration, we can easily obtain $X_1 = U(a\Sigma + b\Sigma^3 + c\Sigma^5)V^T$. Notice that $a\Sigma + b\Sigma^3 + c\Sigma^5$ is also a diagonal matrix, and is in fact the diagonal singular value matrix of $X_1$. So, by applying this step multiple times, only the diagonal matrix of the SVD of $X_0$ changes, while $U$ and $V$ stay the same. Therefore, after applying 5 steps, we obtain a matrix $U D V^T$, where $D$ is some diagonal matrix. So, $O_{t}^{(l)} = \text{Newton-Schulz5}(M_{t}^{(l)}; a,b,c) = U \Sigma' V^T$
\end{proof}

\begin{proposition}
For a single rectangular weight matrix $W^{(l)}_t$ at time $t$, the non-vectorized preconditioner of \textsc{VanillaMuon}($\eta, \mu$) is $D_{t}^{(l)} = U S^{-1} U^T$, where $M_{t}^{(l)} = U \Sigma V^T$ is the SVD of the corresponding momentum matrix and $S = \Sigma \Sigma'^{-1}$ for $O_{t}^{(l)} = U \Sigma' V^T$. being the SVD of the orthogonalized Momentum using Newton-Schulz.
\end{proposition}

\begin{proof}
Let $t > 0$ and $1 \le l \le L$ be an arbitrary time and weight matrix index. Let $M_{t}^{(l)} = U \Sigma V^T$ be the singular value decomposition of the momentum matrix.

The update step of \textsc{VanillaMuon}($\eta, \mu$) is $W_{t+1}^{(l)} = W_t^{(l)} - \eta U\Sigma'V^T$, where $\eta$ is the learning rate. Let $S = \Sigma \Sigma'^{-1}$ and $D_{t}^{(l)} = U S^{-1} U^T$. We then have:

\begin{equation}
D_{t}^{(l)}M_{t}^{(l)} = U S^{-1} U^T U \Sigma V^T = U S^{-1}\Sigma V^T = U \Sigma' V^T
\end{equation}
\end{proof}

\begin{proposition}
Given the non-vectorized preconditioner $D_{t}^{(l)}$ of \textsc{VanillaMuon}($\eta, \mu$), the preconditioner of the vectorized block $w_t^{(l)} = \text{vec}(W_t^{(l)})$ is $P_{t}^{(l)} = I \otimes D_{t}^{(l)}$.
\end{proposition}

\begin{proof}
\begin{equation}
\text{vec}(W^{(l)}_{t+1}) = \text{vec}(W^{(l)}_t - \eta D_{t}^{(l)}M_{t}^{(l)}) = w_t^{(l)} - \eta \text{vec}(D_{t}^{(l)}M_{t}^{(l)}) = w_t^{(l)} - \eta(I \otimes D_{t}^{(l)})m_{t}^{(l)}
\end{equation}
\end{proof}

\begin{proposition}
Given the vectorized preconditioner $P_{t}^{(l)}$ of \textsc{VanillaMuon}, the preconditioner combined for all the weights $\mathcal{W}$ is $P_{t} = \bigoplus_{l=1}^L P_{t}^{(l)}$
\end{proposition}

\begin{proof}
\begin{equation}
\begin{aligned}
    w_{t+1} &= \begin{bmatrix}
    w_{t+1}^{(1)T} & \dots & w_{t+1}^{(L)T} 
    \end{bmatrix}^T = \begin{bmatrix}
    w_{t}^{(1)T} & \dots & w_{t}^{(L)T} 
    \end{bmatrix}^T - \eta \begin{bmatrix}
    (P_{t}^{(1)}m_{t}^{(1)})^T & \dots & (P_{t}^{(L)}m_{t}^{(L)})^T 
    \end{bmatrix}^T \\
    &= w_t - \eta\begin{bmatrix}
    P_{t}^{(1)} & \dots & 0 \\
    \vdots & \ddots & \vdots \\
    0 & \dots & P_{t}^{(L)}
    \end{bmatrix} m_{t} = w_t - \eta \left(\bigoplus_{l=1}^L P_{t}^{(l)}\right)m_{t}
\end{aligned}
\end{equation}
\end{proof}

It is also relevant to show that raising the preconditioner $P_{t}$ to a power $p$ can be done by simply raising the non-vectorized preconditioner matrices $D_{t}^{(l)}$ to $p$. This is important for computing the Effective Sharpness of the matrix $P_{t}^{1/2}H_{t}P_{t}^{1/2}$.

\begin{proposition}
Computing $P_{t}^p$ can be done by simply taking the power of each individual non-vectorized matrix $D_{t}^{(l)}$.
\end{proposition}

\begin{proof}
\begin{equation}
P_{t}^p = \left(\bigoplus_{l=1}^L P_{t}^{(l)}\right)^p = \bigoplus_{l=1}^L P_{t}^{(l)p} = \bigoplus_{l=1}^L (I \otimes D_{t}^{(l)})^p = \bigoplus_{l=1}^L (I \otimes D_{t}^{(l)p})
\end{equation}
\end{proof}

\section{Flat Minima Hypothesis}

\begin{table*}[b]
\caption{For each Learning Rate scheduler (LR) and Optimizer pairing, we report Validation Accuracy (Val Acc), Generalization Gap (Gap), Raw Sharpness (Raw), and Adaptive Sharpness (ASAM) scores. Each entry aggregates the mean and standard deviation across 64 runs.}
\label{tbl:benchmark}
\vskip 0.15in
\begin{center}
\begin{small}
\begin{sc}
\begin{tabular}{llccccc}
\toprule
LR & Optimizer & Train Acc & Val Acc & Acc Gap & Raw ($\times 10^{2}$) & ASAM ($\times 10^{-2}$) \\
\midrule
Fixed & Normalized Muon & 0.97 $\pm$ 0.005 & 0.90 $\pm$ 0.007 & 0.07 $\pm$ 0.007 & 0.69 $\pm$ 0.075 & 0.82 $\pm$ 0.229 \\
Fixed & Decoupled Muon & 0.98 $\pm$ 0.004 & 0.90 $\pm$ 0.006 & 0.08 $\pm$ 0.007 & 0.58 $\pm$ 0.080 & 0.96 $\pm$ 0.307 \\
Fixed & Coupled SGD & 0.87 $\pm$ 0.012 & 0.80 $\pm$ 0.057 & 0.07 $\pm$ 0.051 & 0.62 $\pm$ 0.135 & 2.06 $\pm$ 0.911 \\
Fixed & Coupled Adam & 0.91 $\pm$ 0.007 & 0.87 $\pm$ 0.012 & 0.04 $\pm$ 0.014 & 1.34 $\pm$ 0.266 & 1.93 $\pm$ 0.626 \\
\midrule
LDS & Normalized Muon & 1.00 $\pm$ 0.001 & 0.94 $\pm$ 0.002 & 0.06 $\pm$ 0.002 & 0.45 $\pm$ 0.078 & 0.71 $\pm$ 0.097 \\
LDS & Decoupled Muon & 1.00 $\pm$ 0.001 & 0.93 $\pm$ 0.002 & 0.07 $\pm$ 0.002 & 0.13 $\pm$ 0.076 & 0.70 $\pm$ 0.283 \\
LDS & Coupled SGD & 0.98 $\pm$ 0.003 & 0.92 $\pm$ 0.002 & 0.06 $\pm$ 0.003 & 0.03 $\pm$ 0.005 & 0.32 $\pm$ 0.072 \\
LDS & Coupled Adam & 0.99 $\pm$ 0.002 & 0.92 $\pm$ 0.002 & 0.07 $\pm$ 0.003 & 1.05 $\pm$ 0.112 & 0.56 $\pm$ 0.074 \\
\bottomrule
\end{tabular}
\end{sc}
\end{small}
\end{center}
\vskip -0.1in
\end{table*}

\begin{figure*}[ht]
\vskip 0.2in
\begin{center}
\begin{subfigure}
    \centering
    \includegraphics[width=0.98\textwidth]{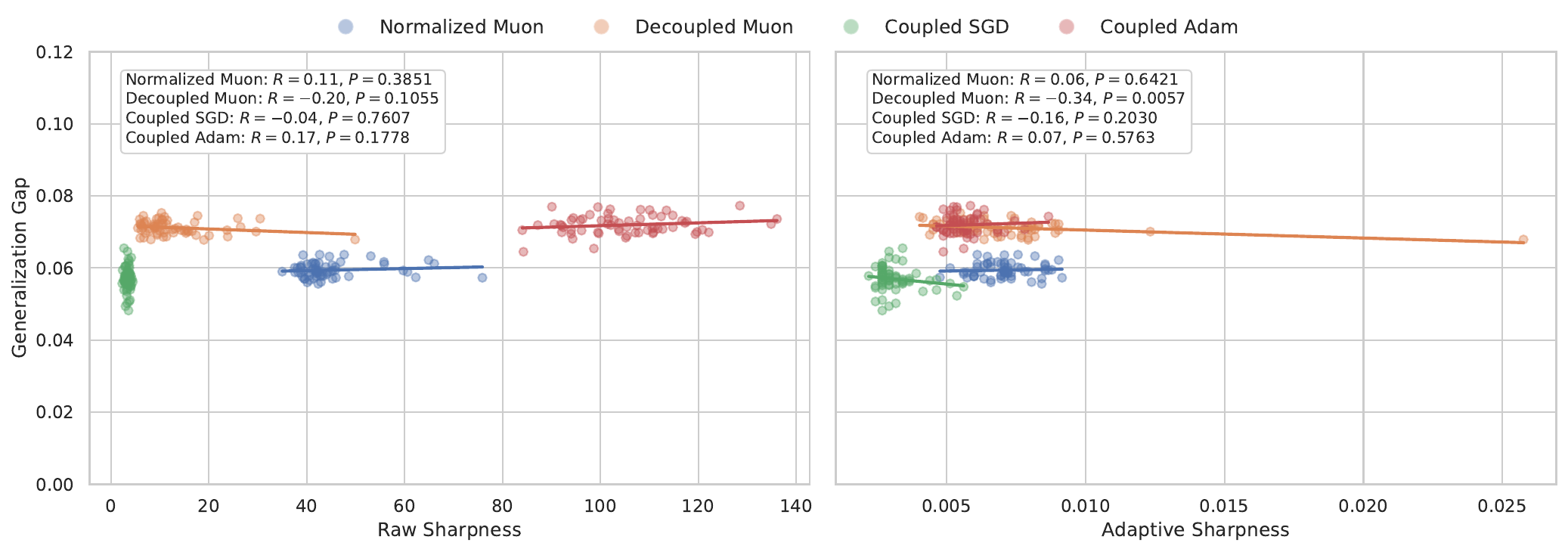}
    \label{fig:scheduled_sharp_gap}
\end{subfigure}
\caption{Per-optimizer correlation of raw/adaptive sharpness with generalization gap using a LDS.}
\end{center}
\vskip -0.2in
\end{figure*}

\section{Edge of Stability Graphs}

In this section, we provide the plots for the AEoS of \textsc{VanillaAdam} and \textsc{RMSprop}.

\begin{figure}[ht]
\vskip 0.2in
\begin{center}
\begin{subfigure}
    \centering
    \includegraphics[width=0.395\columnwidth]{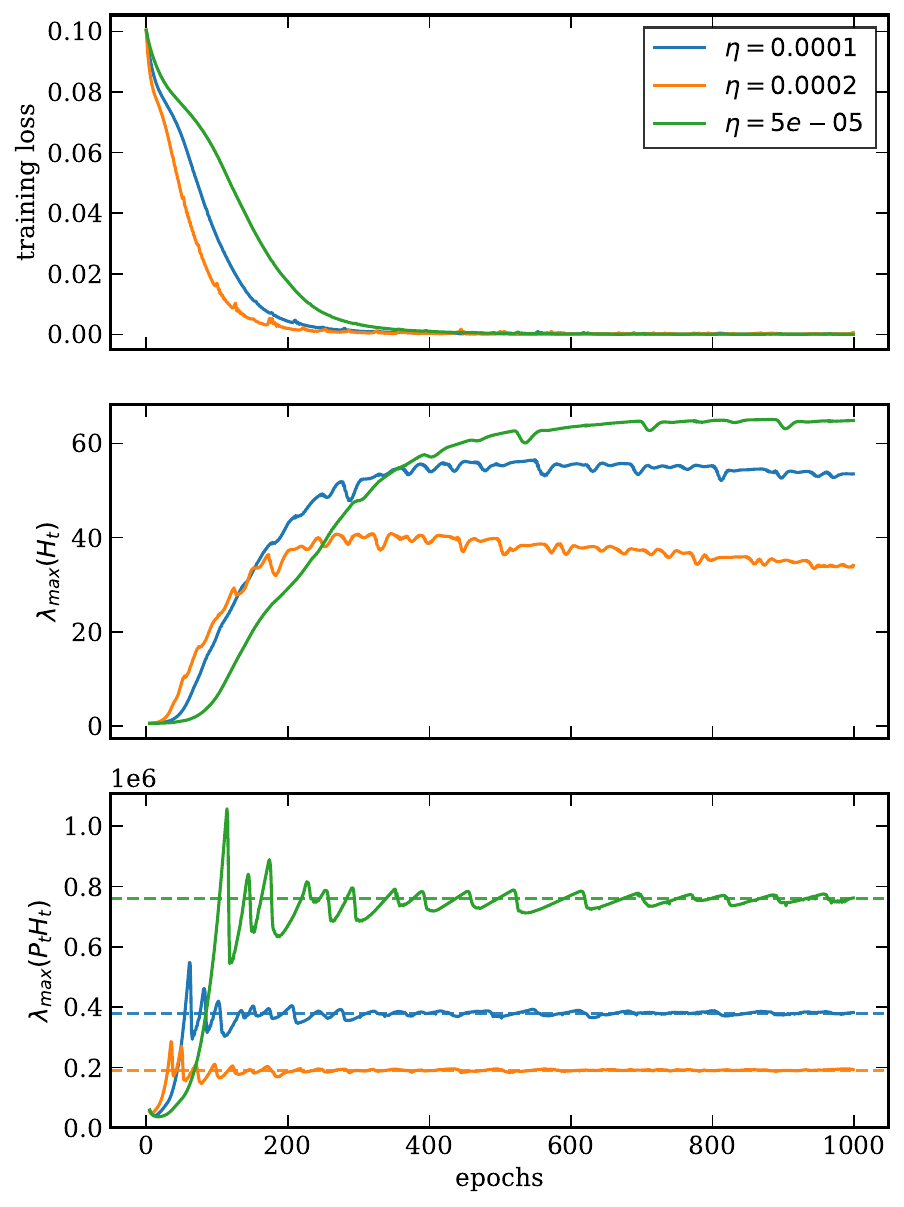}
    \label{fig:adam_eos}
\end{subfigure}
\begin{subfigure}
    \centering
    \includegraphics[width=0.4\columnwidth]{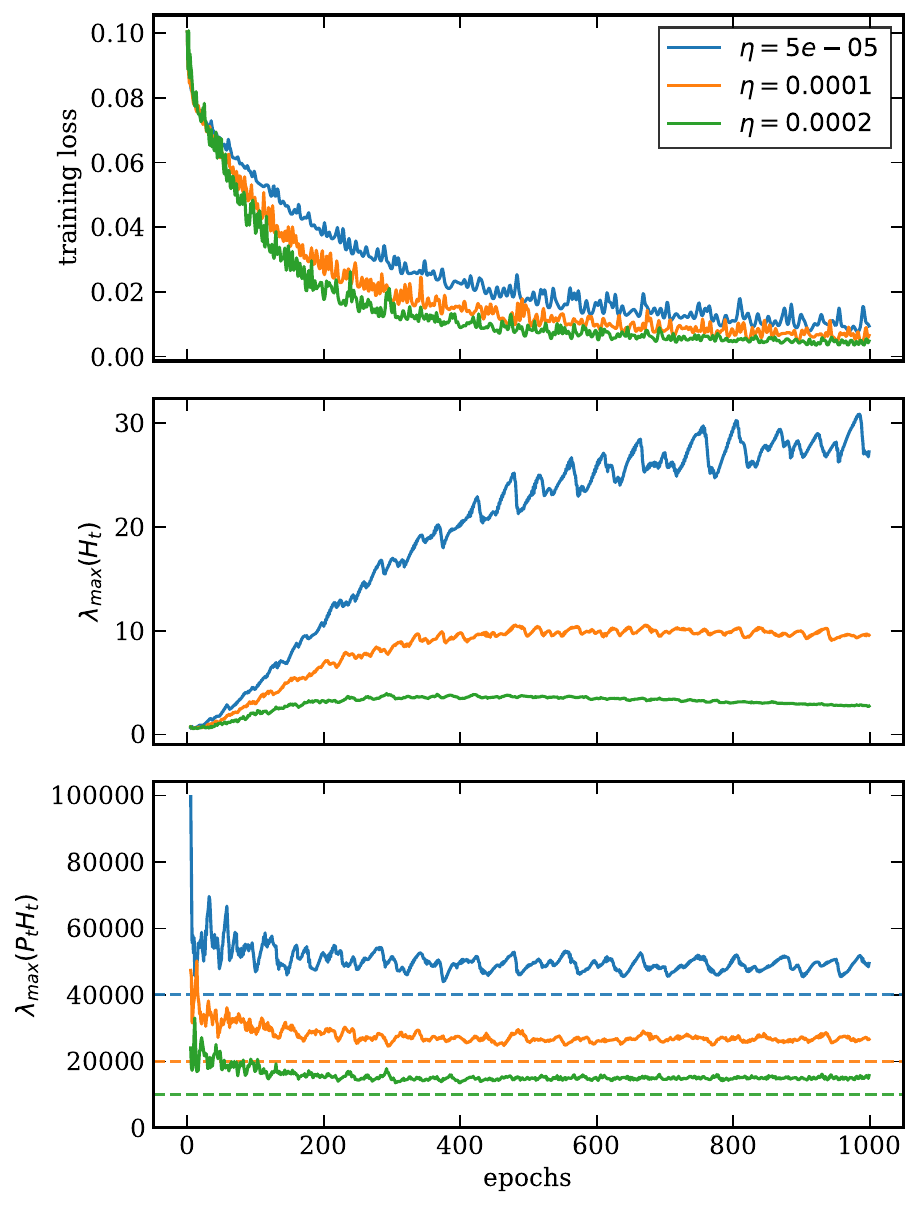}
    \label{fig:rmsprop_eos}
\end{subfigure}
\caption{The training loss (top), raw sharpness (middle) and Effective Sharpness (bottom) of \textsc{VanillaAdam}($\eta, \beta_1, \beta_2$) with $\beta_1 = 0.9, \beta_2 = 0.999$ and three different $\eta$ (left) and \textsc{RMSprop}($\eta, \alpha$) (right) with $\alpha = 0.99$ and three different $\eta$. The dashed lines in the bottom plots represent the threshold where it is expected for the algorithms to oscillate, computed as $2/\eta$ for \textsc{RMSprop}($\eta, \alpha$) and $(2 + 2\beta_1)/\eta$ for \textsc{VanillaAdam}($\eta, \beta_1, \beta_2$) \cite{goh2017momentum}. \textsc{VanillaAdam} oscillates around the threshold and \textsc{RMSprop} slightly above the threshold, exactly as observed in previous work \cite{cohen2024adaptivegradientmethodsedge}}
\end{center}
\end{figure}

\section{Sharpness Computation Details}

To compute the raw and Effective Sharpness in our experiments we utilize a power iteration approach inspired by PyHessian \cite{yao2020pyhessian}. In addition to their implementation, we add support for a preconditioner to compute the effective Hessian sharpness. The algorithm can be found in \ref{alg:sharpness}

\begin{algorithm}[ht]
   \caption{Sharpness Power Iteration}
   \label{alg:sharpness}
\begin{algorithmic}
   \STATE {\bfseries Input:} preconditioner $P_t$, Hessian $H_t$, maximum iterations $n_\text{max}$, tolerance $\tau$
   \STATE Initialize $v_0 \sim \mathcal{N}(0, I_p)$.
   \STATE Normalize $v_1 = v_0 / \Vert v_0 \Vert_2$
   \STATE Initialize $\lambda_\text{max} = \text{none}$
   \FOR{$i=1$ {\bfseries to} $n_\text{max}$}
   \STATE $Mv_i = P_t^{1/2}H_tP_t^{1/2}v_i$
   \STATE $\lambda_\text{temp} = v_i^TMv_i$
   \STATE $v_{i+1} = Mv_i / \Vert Mv_i \Vert_2$
   \IF{$\lambda_\text{max} = \text{none}$}
   \STATE $\lambda_\text{max} = \lambda_\text{temp}$
   \ELSIF{$|\lambda_\text{max} - \lambda_\text{temp}| / |\lambda_\text{max} + 10^{-7}| < \tau$}
   \STATE End for loop
   \ENDIF
   \STATE $\lambda_\text{max} = \lambda_\text{temp}$
   \ENDFOR
   \STATE \textbf{return} $\lambda_\text{max}$
\end{algorithmic}
\end{algorithm}

Note that in Algorithm \ref{alg:sharpness}, it is not relevant to store the entire Hessian $H_t$ and preconditioner $P_t$, as $H_tv$ can be computed efficiently via the Pearlmutter trick and $P^{1/2}_tv$ can be computed very efficiently if it is diagonal and quite efficiently if it is block diagonal. To compute the raw sharpness, all that has to be done is to pass $P_t = I$. In our computations, we used a tolerance of $\tau = 10^{-9}$ and a maximum number of iterations $n_\text{max} = 200$.

\section{Optimizer Definitions} \label{app:optims}

\subsection{Muon Optimizer}
The Muon optimizer \cite{jordan2024muon} is defined as follows: At time $t$, for a weight matrix $W_t^{(l)}$ with index $l$ of a neural network with weights $w_t \in \mathbb{R}^p$ and loss function $\mathcal{L}: \mathbb{R}^p \rightarrow \mathbb{R}$, the update is computed as:

\begin{equation}
\begin{aligned}
    M_{t}^{(l)} &= \mu M_{t-1}^{(l)} + \nabla_{W_t^{(l)}} \mathcal{L}(w_t) \\
    O_{t}^{(l)} &= \text{Newton-Schulz5}(M_{t}^{(l)}; a, b, c) \\
    W_{t+1}^{(l)} &= W_t^{(l)} - \eta_t O_{t}^{(l)}
\end{aligned}
\end{equation}
We refer to this rule as \textsc{VanillaMuon}($\eta_t, \mu$). In \citet{liu2025muonscalableLLM}, the authors extend Muon by adding decoupled weight decay. The update rule then becomes $W_{t+1}^{(l)} = W_t^{(l)} - \eta_t(O_t^{(l)} + \lambda W_t^{(l)})$, where $\lambda$ is the weight decay constant. We refer to this rule as \textsc{DecoupledMuon}($\eta_t, \mu, \lambda$) since it uses decoupled weight decay. The final update used in Airbench \citet{jordan2024airbench} involves performing a normalization step before the update. I.e., before the new weights at $t+1$ are computed, the weights are normalized as $W_t^{(l)} = \sqrt{m_l}W_t^{(l)} / \Vert W_t^{(l)}\Vert_F$, where $W_t^{(l)} \in \mathbb{R}^{m_l \times n_l}$. We refer to this update rule as \textsc{NormalizedMuon}($\eta_t, \mu$).

\subsection{Other Optimizers}

We compare Muon to several other optimizers, including a non-adaptive optimizer and a diagonally preconditioned optimizer. We define Gradient Descent with weight decay as \textsc{CoupledGD}($\eta_t, \lambda$) with coupled weight decay.

The second algorithm we use is Adam, which we refer to as \textsc{VanillaAdam}($\eta_t, \beta_1, \beta_2$) with learning rate $\eta_t$ at time $t$ and the momentum ($\beta_1$) and squared exponential average ($\beta_2$) coefficients. As with Muon, we also use a coupled weight decay version of Adam in our experiments, named \textsc{CoupledAdam}($\eta_t, \beta_1, \beta_2, \lambda$) with the weight decay constant $\lambda$.

Finally, we use RMSprop, named \textsc{RMSprop}($\eta_t, \alpha$) with learning rate $\eta_t$ at time $t$ and squared gradient exponential average coefficient $\alpha$. Since this algorithm is only used for AEoS analysis, we do not use a weight decay alternative.

\end{document}